\documentclass{article} 
\usepackage{iclr2027_conference,times}

\usepackage{amsmath,amsfonts,bm}

\def\eqref#1{equation~\ref{#1}}

\def\1{\bm{1}}

\DeclareMathAlphabet{\mathsfit}{\encodingdefault}{\sfdefault}{m}{sl}
\SetMathAlphabet{\mathsfit}{bold}{\encodingdefault}{\sfdefault}{bx}{n}

\usepackage[hidelinks]{hyperref}
\usepackage{url}

\usepackage{amsmath,amsthm}
\usepackage{amssymb}
\usepackage{booktabs}
\usepackage{multirow}
\usepackage{tabularx}
\usepackage{array}
\usepackage[table]{xcolor}
\usepackage{graphicx}
\usepackage{algorithm}
\usepackage{algpseudocode}

\usepackage[most]{tcolorbox}
\usepackage{listings}
\usepackage{xcolor}
\algrenewcommand\algorithmiccomment[1]{
    \hfill$\triangleright$~\textit{#1}
}

\definecolor{promptgreen}{HTML}{238B2A}
\definecolor{promptbackground}{HTML}{F2FAF2}
\definecolor{promptorange}{HTML}{C65D12}
\definecolor{promptbackground2}{HTML}{FFF7ED}
\lstdefinestyle{promptstyle}{
    basicstyle=\ttfamily\footnotesize,
    columns=fullflexible,
    keepspaces=true,
    breaklines=true,
    breakatwhitespace=true,
    breakautoindent=false,
    breakindent=0pt,
    prebreak={},
    postbreak={},
    showstringspaces=false,
    showtabs=false,
    showspaces=false,
    tabsize=2,
    aboveskip=0pt,
    belowskip=0pt,
    xleftmargin=0pt,
    xrightmargin=0pt
}

\newtcblisting{promptbox}[2][]{
    enhanced,
    breakable,
    listing only,
    listing engine=listings,
    listing options={style=promptstyle},
    title={#2},
    colback=promptbackground,
    colframe=promptgreen,
    colbacktitle=promptgreen,
    coltitle=white,
    fonttitle=\bfseries\normalsize,
    boxrule=0.7pt,
    arc=1.5mm,
    outer arc=1.5mm,
    boxsep=0pt,
    left=5mm,
    right=5mm,
    top=4mm,
    bottom=4mm,
    lefttitle=4mm,
    righttitle=4mm,
    toptitle=1.5mm,
    bottomtitle=1.5mm,
    before skip=6pt,
    after skip=10pt,
    #1
}
\newtcblisting{agentpromptbox}[2][]{
    enhanced,
    breakable,
    listing only,
    listing engine=listings,
    listing options={style=promptstyle},
    title={#2},
    colback=promptbackground2,
    colframe=promptorange,
    colbacktitle=promptorange,
    coltitle=white,
    fonttitle=\bfseries\normalsize,
    boxrule=0.7pt,
    arc=1.5mm,
    outer arc=1.5mm,
    boxsep=0pt,
    left=5mm,
    right=5mm,
    top=4mm,
    bottom=4mm,
    lefttitle=4mm,
    righttitle=4mm,
    toptitle=1.5mm,
    bottomtitle=1.5mm,
    before skip=6pt,
    after skip=10pt,
    #1
}

\newtheorem{proposition}{Proposition}
\title{When All Rollouts Fail: Policy-Compatible Joint Decoding for Hard-Problem Reasoning}

\newif\ifpublicversion
\publicversiontrue
\ifpublicversion
\iclrfinalcopy
\author{
\textbf{Zehao Wang}\textsuperscript{1,2}
\hspace{0.55em}
\textbf{Hanzhu Chen}\textsuperscript{2}
\hspace{0.55em}
\textbf{Runquan Gui}\textsuperscript{1,2}
\hspace{0.55em}
\textbf{Zhe Wang}\textsuperscript{1}
\hspace{0.55em}
\textbf{Lin Yang}\textsuperscript{3}\thanks{Corresponding author:
\texttt{lin.yang.jason@gmail.com}. Work done during Zehao Wang's internship at
Tencent. Contact: \texttt{zh-wang@mail.ustc.edu.cn}.}
\\[0.4ex]
{\normalfont\textsuperscript{1} University of Science and Technology of China}
\hspace{0.55em}
{\normalfont\textsuperscript{2} Tencent}
\\
{\normalfont\textsuperscript{3} Shanghai Artificial Intelligence Laboratory }
}
\fi
\newcommand{\Mymethod}{\textsc{PICO}}
\begin{document}

\maketitle
\ifpublicversion\lhead{Preprint. Under review.}\fi

\begin{abstract}
Reinforcement Learning with Verifiable Rewards (RLVR) derives its learning signal from binary rewards via on-policy exploration and becomes a powerful paradigm for improving the reasoning of large language models. Yet it faces a learning cliff on problems beyond model's current ability: when every rollout fails, the samples reveal nothing about a correct solution, and group-relative estimators (e.g., GRPO, RLOO) collapse the update to zero advantage. Expert trajectories offer a remedy, but they are incompatible with the current policy and often lie in the student's low-probability regions, so the surviving updates demand large parameter shifts and decay into surface-level imitation rather than transferable reasoning. We argue that a useful learning signal on hard problems must be both verifiably correct and sufficiently compatible with the current policy.
We propose Policy-Compatible Reinforcement Learning (\textbf{\Mymethod{}}), which performs distribution-constrained joint decoding: expert data guides decoding at the token level while preserving the student's relative token probabilities, and a dynamic top-k mechanism confines the sampling distribution under a finite divergence budget, yielding trajectories that are verified and policy-compatible by construction. These trajectories are optimized jointly with on-policy rollouts under a mixed-behavior objective that mitigates their distribution mismatch via importance sampling. Experiments across reasoning diverse datasets and agentic benchmarks demonstrate that \Mymethod{} consistently outperforms standard RLVR and prior off-policy baselines.
\end{abstract}

\section{Introduction}
\label{introduction}

Reinforcement learning with verifiable rewards (RLVR) has emerged as a prominent paradigm for post-training large language models (LLMs)\citep{guo2025deepseek,jaech2024openai,team2025kimi}, demonstrating strong effectiveness in enhancing reasoning capabilities across domains such as mathematics, coding, and agentic tasks\citep{dong2026agentic,zhang2025agentlandscape}. RLVR enables LLMs to learn directly from verifiable reward feedback on trajectories generated through on-policy exploration, yielding substantial and generalizable performance gains\citep{zhu2025path,shenfeld2026rlrazor}. However, this reliance on successful exploration creates a learning cliff on challenging problems beyond the model’s current capabilities: when all sampled rollouts fail, the policy receives no positive reward signal\citep{agrawal2026offgrpo,zhang2025critiquegrpo,yue2025does}. This issue is particularly acute for methods based on group-relative advantage estimation (e.g., GRPO)\citep{yu2026dapo,zhang2026scaf}, as all-zero-reward groups cause the estimated advantages to collapse to zero, leaving no informative learning signal to drive futher improvement on these problems.

A natural remedy is to leverage expert data, which offer a promising way to overcome this exploration bottleneck and are increasingly used to supplement online RL training. Typically generated by human experts or stronger models, expert trajectories expose the student to high-quality reasoning paths and provide informative learning signals. Yet these trajectories are not necessarily compatible with the current policy and often fall within low-probability regions under its distribution.
Existing methods utilize such trajectories in three main ways. Some use them as supervised fine-tuning (SFT) targets alongside RL training\citep{ma2026relift,fu2026srft,zhang2026chord}. When these trajectories lie in low-probability regions under the current policy, directly fitting them can induce excessive policy drift and result in surface-level imitation rather than transferable reasoning\citep{yan2025luffy}. Other approaches incorporate expert trajectories into rollout groups as off-policy samples to recover informative learning signals\citep{yan2025luffy,dong2026rlplus}. The distributional mismatch between these off-policy trajectories and on-policy rollouts can distort the resulting learning signals and destabilize optimization\citep{agrawal2026offgrpo,setlur2026reuse,shenfeld2026rlrazor}. A further line of work\citep{huang2026prefixrft,qu2026pope,zhang2026stephint,zhang2026bread} leverages expert trajectories indirectly, using them to construct prefixes or hints that guide subsequent exploration and generation. Although the resulting continuations are generated by the current policy, they are conditioned on intermediate states that the policy may not be able to reach on its own from the original prompt\citep{ross2011dagger,qu2026pope,agrawal2026offgrpo}. This creates privileged-context dependence: once the guidance is removed, the model may struggle to reproduce the full reasoning process. 
Collectively, these limitations reveal a fundamental trade-off: \textit{an effective learning signal on challenging problems must come from trajectories that are both verifiably correct and sufficiently compatible with the current student policy.} Although expert data can provide verifiably correct trajectories, naively incorporating them may sacrifice policy compatibility by pulling the policy toward behaviors that lie far outside its current distribution, ultimately undermining stable and effective policy improvement.
To address this challenge, we propose Policy-Compatible Reinforcement Learning (\Mymethod{}), which uses distribution-constrained joint decoding to recover positive learning signals from off-policy trajectories while explicitly controlling the induced distribution shift.
Specifically, joint decoding constructs the next-token candidate set by intersecting high-probability tokens predicted by a teacher conditioned on expert data with those predicted by the student, thereby restricting sampling to candidates favored by both models. In this way, the teacher provides token-level guidance to improve rollout accuracy, while sampling according to the student’s relative probabilities preserves its own preferences. To further constrain the joint sampling distribution within a finite divergence budget, we introduce a dynamic top-$k$ mechanism that adaptively adjusts the teacher’s candidate set.
Finally, the resulting trajectories verified for correctness are optimized alongside on-policy rollouts through a mixed-behavior objective.

As illustrated in Figure~\ref{fig:picorl-overview}, \Mymethod{} activates joint decoding only for all-failure groups, preserving standard on-policy optimization whenever student exploration already yields positive rewards. Figure~\ref{fig:rescue-dynamics} further shows that joint decoding recovers verified, positive-reward trajectories in up to 80\% of all-failure groups, thereby restoring informative learning signals where standard on-policy sampling yields none. Extensive experiments across mathematics and general reasoning benchmarks demonstrate that \Mymethod{} consistently outperforms standard RLVR and existing expert-guided baselines. Further evaluations on agentic tasks establish its generalizability beyond conventional reasoning settings. Comprehensive ablation studies validate the effectiveness of its key components and confirm the importance of explicitly controlling the trade-off between rollout quality and student-policy proximity.
In summary, our main contributions are threefold:
\begin{itemize}
\item We identify a fundamental trade-off between rollout quality and student-policy proximity when incorporating offline expert information into RLVR.

\item We introduce \Mymethod{}, which combines distribution-constrained joint decoding with failure-triggered mixed-behavior optimization to generate and learn from verified, policy-compatible trajectories.

\item Extensive experiments on mathematics, general reasoning, and agentic tasks demonstrate the effectiveness and generalizability of \Mymethod{}.

\end{itemize}

\section{Preliminaries}
\label{sec:preliminaries}

\textbf{Reinforcement Learning with Verifiable Rewards.}
Let $\mathcal{D}$ denote a training distribution over queries
$x\in\mathcal{X}$. Given a query $x$, the policy $\pi_\theta$ generates a
response $y=(y_1,\ldots,y_T)$ autoregressively, where $T$ is the response
length. Let $a(y)$ denote the answer extracted
from $y$ and $a^\star(x)$ the ground-truth answer to $x$. The verifier assigns
the binary reward
$
R(x,y)=\mathbf{1}\!\left\{a(y)=a^\star(x)\right\}.
$
The RLVR objective maximizes the expected verifier reward\citep{guo2025deepseek}:
\begin{equation}
J(\theta)
=
\mathbb{E}_{x\sim\mathcal{D}}
\mathbb{E}_{y\sim\pi_\theta(\cdot\mid x)}
\left[R(x,y)\right].
\label{eq:rl-objective}
\end{equation}

\textbf{Group Relative Policy Optimization.}
To optimize this objective, policy-gradient methods are commonly employed\citep{schulman2017ppo}. Among them, Group Relative Policy Optimization (GRPO)~\citep{shao2024deepseekmath} is a representative approach that converts binary rewards into relative advantages within each rollout group without requiring a separately trained value model. For each query $x$, the behavior policy $\pi_{\mathrm{old}}$
generates a group of trajectories $\{y^{(i)}\}_{i=1}^{n}$, which are scored
to obtain rewards $\{R^{(i)}\}_{i=1}^{n}$. The advantage of each trajectory
is then computed by standardizing its reward within the group:
\begin{equation}
A^{(i)}
=
\frac{
R^{(i)}
-
\operatorname{mean}\!\left(\{R^{(j)}\}_{j=1}^{n}\right)
}{
\operatorname{std}\!\left(\{R^{(j)}\}_{j=1}^{n}\right)
}.
\label{eq:grpo-advantage}
\end{equation}
Recent work\citep{liu2025drgrpo} suggests that the standard deviation term in the advantage
calculation may introduce bias into policy optimization. 
The policy is updated by optimizing the clipped, KL-regularized surrogate averaged
over training problems:
\begin{equation}
\mathcal{J}_{\mathrm{GRPO}}(\theta)
=
\mathbb{E}_{x\sim\mathcal{D}}
\left[
\frac{1}{\mathcal{Z}}
\sum_{i=1}^{n}
\sum_{t=1}^{|y^{(i)}|}
\min\!\left(
r_{i,t}(\theta)A^{(i)},
\operatorname{clip}\!\left(
r_{i,t}(\theta),1-\epsilon,1+\epsilon
\right)A^{(i)}
\right)
-
\beta\mathcal{L}_{\mathrm{KL}}(\theta)
\right].
\label{eq:grpo}
\end{equation}
Here, $\mathcal{Z}=\sum_{i=1}^{n}|y^{(i)}|$ is the total number of response
tokens in the rollout group, and
$r_{i,t}(\theta)=
\pi_\theta(y_t^{(i)}\mid x,y_{<t}^{(i)})/
\pi_{\mathrm{old}}(y_t^{(i)}\mid x,y_{<t}^{(i)})$
is the importance-sampling ratio.  The clipping operation limits excessively
large policy updates, while the KL term softly constrains $\pi_\theta$ to
remain close to the reference policy $\pi_{\mathrm{ref}}$. In practice, this
KL penalty can be downweighted or omitted when larger policy shifts are
desired\citep{yan2025luffy,zhang2025critiquegrpo}.

\begin{figure*}[t]
    \centering
    \includegraphics[
        width=0.98\textwidth
    ]{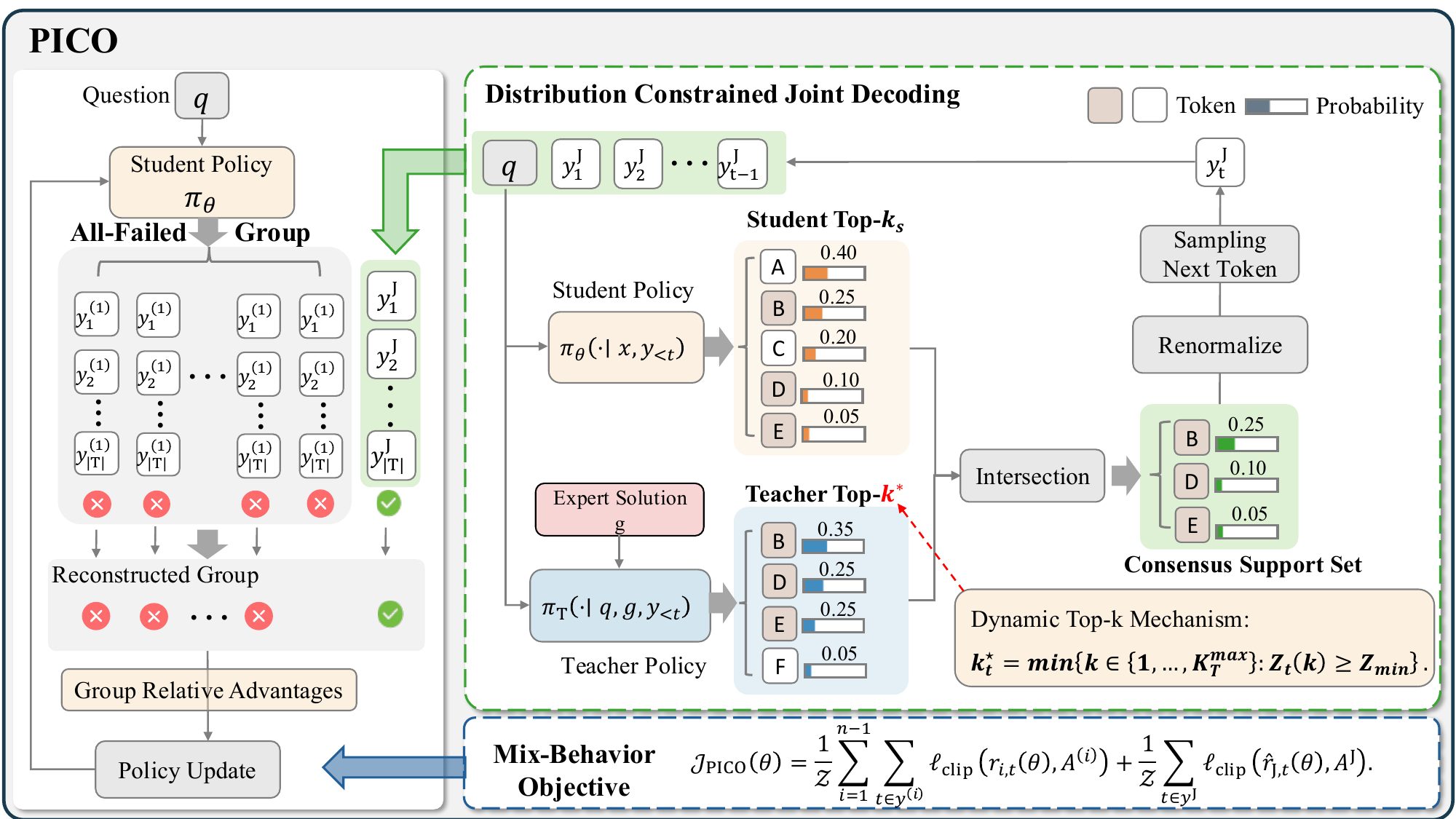}
    \caption{
    Overview of \textbf{PICO}. Given an all-incorrect rollout group sampled from the student policy, PICO generates a verified joint trajectory to replace one failed rollout for mixed-behavior policy optimization. The joint trajectory is generated token by token through distribution-constrained joint decoding, which intersects the teacher and student top-$k$ candidate sets. 
    }
    \label{fig:picorl-overview}
\end{figure*}

\section{Method}
\label{method}
In this section, we detail the design and implementation of \Mymethod. As illustrated in Figure \ref{fig:picorl-overview} and Algorithm \ref{alg:picorl}, the framework consists of two components: distribution-constrained joint decoding and mixed-behavior  policy optimization. First, distribution-constrained joint decoding combines the token-level preferences of the teacher and student to generate rollouts, with a dynamic top-$k$ mechanism adaptively balancing teacher guidance and distribution shift. Second, mixed-behavior policy optimization learns from the resulting trajectories using importance-sampling (IS) ratios to account for the mismatch between the joint-decoding policy and the student policy.

\subsection{Distribution-Constrained Joint Decoding}
\label{subsec:joint-decoding}

To improve the reliability of sampled trajectories, we employ a
teacher--student joint-decoding strategy that restricts token selection to
candidates jointly favored by both models. We first construct a consensus
support set and then optimize a behavior policy over this set, thereby balancing rollout quality against proximity to the student policy.

\textbf{Consensus Support Set.}
Given a query $x$, privileged expert guidance $g$ available only to the
teacher, and a shared generated prefix $y_{<t}$, the next-token distributions
of the student and the teacher models are defined as
\begin{equation}
p_t(v)
=
\pi_{\mathrm{old}}(v\mid x,y_{<t}),
\qquad
q_t(v)
=
\pi_{\mathrm{T}}(v\mid x,g,y_{<t}),
\label{eq:student-teacher-distributions}
\end{equation}
where $v$ denotes a token in the shared vocabulary. Both models condition on
the same generated prefix, while only the teacher has access to $g$. Thus,
the privileged guidance is neither exposed to the student nor appended to
its rollout context; it influences generation only through the teacher's
token-level predictions.

At each decoding step, we construct the consensus support set by intersecting
the top-$k$ sets of the two models:
\begin{equation}
F_t
=
\operatorname{TopK}(q_t,k_{\mathrm T})
\cap
\operatorname{TopK}(p_t,k_{\mathrm S}),
\label{eq:joint-candidate-set}
\end{equation}
where $k_{\mathrm T}$ and $k_{\mathrm S}$ denote the teacher and student
cutoffs, respectively. By retaining only tokens
supported by both models, $F_t$ filters out predictions that are implausible
under either distribution. The teacher contributes privileged information
through candidate selection, while the student continues to determine which
of the retained tokens are compatible with its current policy.

\textbf{Joint-Decoding Policy Optimization.}
Based on a nonempty candidate set $F_t$, the remaining challenge is to assign
sampling probabilities within this set. The resulting behavior policy must
balance two competing objectives. On the one hand, it should assign greater
probability to tokens favored by the teacher so that privileged guidance can
improve rollout quality. On the other hand, it should remain close to the
rollout-time student distribution, thereby keeping the joint trajectory compatible to student optimization. 

We quantify these objectives using the expected teacher confidence
$\mathbb{E}_{v\sim\mu}[q_t(v)]$ and the reverse KL divergence
$D_{\mathrm{KL}}(\mu\Vert p_t)$, respectively. This yields the following
joint-decoding objective:
\begin{equation}
\mu_t^\lambda
=
\arg\min_{\mu\in\Delta(F_t)}
\left[
D_{\mathrm{KL}}(\mu\Vert p_t)
-
\lambda\,\mathbb{E}_{v\sim\mu}\!\left[q_t(v)\right]
\right],
\label{eq:joint-policy-optimization}
\end{equation}
where $\Delta(F_t)$ is the set of probability distributions supported on
$F_t$, and $\lambda\geq0$ controls the trade-off between teacher guidance and
proximity to the student policy. A smaller $\lambda$ better preserves the
student's relative token preferences within $F_t$, whereas a larger
$\lambda$ gives greater influence to the teacher. 
We describe how to obtain this target behavior policy
in Section~\ref{subsec:dynamic-topk}.
\subsection{Dynamic Top-$k$ Joint Decoding}
\label{subsec:dynamic-topk}

\textbf{Closed-Form Behavior Policy.}
For a nonempty support set $F_t$, the optimization problem in
Eq.~\eqref{eq:joint-policy-optimization} admits the following closed-form
solution:
\begin{equation}
\mu_t^\lambda(v)
=
\frac{
p_t(v)\exp\!\left(\lambda q_t(v)\right)
\mathbf{1}\{v\in F_t\}
}{
C_t(\lambda)
},
\qquad
C_t(\lambda)
=
\sum_{u\in F_t}
p_t(u)\exp\!\left(\lambda q_t(u)\right),
\label{eq:soft-joint-policy}
\end{equation}
where $C_t(\lambda)$ is the normalizing constant. The solution incorporates
teacher guidance through two complementary mechanisms: $F_t$ imposes a hard
support constraint by retaining only tokens selected by both models, while
$\exp(\lambda q_t(v))$ softly reweights the student probabilities within
this support. The distribution shifts induced by support restriction and soft reweighting
can be separated exactly.

\begin{proposition}[KL divergence decomposition]
\label{prop:kl-decomposition}
Let
$
Z_t=\sum_{v\in F_t}p_t(v)
\label{eq:retained-mass}
$
denote the student probability mass retained by $F_t$, and let $\mu_t^0$
denote the policy in Eq.~\eqref{eq:soft-joint-policy} at $\lambda=0$. Then
\begin{equation}
D_{\mathrm{KL}}(\mu_t^\lambda\Vert p_t)
=
-\log Z_t
+
D_{\mathrm{KL}}(\mu_t^\lambda\Vert\mu_t^0)
\geq
-\log Z_t.
\label{eq:kl-decomposition}
\end{equation}
The lower bound is attained at $\lambda=0$, for which
$\mu_t^\lambda=\mu_t^0$.
\end{proposition}

The first term, $-\log Z_t$, measures the shift caused by restricting the
sampling support to $F_t$, whereas the second term measures the additional
shift introduced by soft teacher reweighting. Therefore, for a fixed support
set $F_t$, the distribution shift is minimized at $\lambda=0$. Empirically,
we find that the support constraint alone achieves sufficient rollout
accuracy, making additional teacher-based reweighting unnecessary. We
therefore set $\lambda=0$ and obtain our final hard joint-decoding policy:
\begin{equation}
\mu_t^{\mathrm H}(v)
=
\mu_t^0(v)
=
\frac{
p_t(v)\mathbf{1}\{v\in F_t\}
}{
Z_t
}.
\label{eq:hard-joint-policy}
\end{equation}
This policy preserves the student's relative token probabilities within
$F_t$, while the teacher continues to guide generation through support
selection. The derivation of
Eq.~\eqref{eq:soft-joint-policy} and the proof of
Proposition~\ref{prop:kl-decomposition} are provided in Appendix\ref{theoretical-proofs}.

\textbf{Dynamic Top-$k$ Selection.}
Under hard joint decoding, the per-step divergence is determined exactly by
the retained student probability mass,
$D_{\mathrm{KL}}(\mu_t^{\mathrm H}\Vert p_t)=-\log Z_t$.
Therefore, imposing $Z_t\geq Z_{\min}$ guarantees
$D_{\mathrm{KL}}(\mu_t^{\mathrm H}\Vert p_t)\leq-\log Z_{\min}$. To enforce
this constraint adaptively, we fix the student cutoff $K_{\mathrm S}$ and,
for each candidate teacher cutoff $k$, define
\begin{equation}
F_t(k)
=
\operatorname{TopK}(q_t,k)
\cap
\operatorname{TopK}(p_t,K_{\mathrm S}),
\qquad
Z_t(k)
=
\sum_{v\in F_t(k)}p_t(v).
\label{eq:dynamic-candidate-set}
\end{equation}
We select the smallest teacher cutoff whose consensus support retains at
least $Z_{\min}$ student probability mass:
\begin{equation}
k_t^\star
=
\min
\left\{
k\in\{1,\ldots,K_{\mathrm T}^{\max}\}
:
Z_t(k)\geq Z_{\min}
\right\}.
\label{eq:dynamic-teacher-cutoff}
\end{equation}
In practice, we increase $k$ from $1$ to $K_{\mathrm T}^{\max}$ and stop at the first value satisfying the retained-
mass constraint. We then set $F_t=F_t(k_t^\star)$ and
$Z_t=Z_t(k_t^\star)$
If no feasible cutoff exists, we fall back to the student policy by setting
$\mu_t=p_t$ for that decoding step. 

\subsection{Mixed-Behavior Policy Optimization}
\textbf{Failure-Triggered Group Reconstruction.}
Following the GRPO procedure described in
Section~\ref{sec:preliminaries}, we invoke joint decoding only when all $n$
trajectories in a rollout group are incorrect. we sample one additional trajectory $y^{\mathrm J}$
using the hard joint-decoding policy. If $y^{\mathrm J}$ passes verification,
it replaces one incorrect trajectory in the original group.
Without loss of generality, let $y^{(n)}$ be the replaced trajectory. The
resulting mixed-behavior group and its reward set are
\begin{equation}
\widetilde{G}
=
\{y^{(1)},\ldots,y^{(n-1)},y^{\mathrm J}\},
\qquad
\widetilde{\mathcal R}
=
\{R^{(1)},\ldots,R^{(n-1)},R^{\mathrm J}\}.
\label{eq:reconstructed-group}
\end{equation}
For a successfully reconstructed group, the advantage of each trajectory is
computed as
\begin{equation}
A^{(i)}
=
\widetilde{R}^{(i)}
-
\operatorname{mean}\!\left(\widetilde{\mathcal R}\right),
\label{eq:mixed-group-advantage}
\end{equation}
where $\widetilde{R}^{(i)}=R^{(i)}$ for $i<n$,
$\widetilde{R}^{(n)}=R^{\mathrm J}$, and $A^{\mathrm J}:=A^{(n)}$. The
verified joint trajectory thus restores reward variation to a group that
would otherwise provide no group-relative learning signal.

\textbf{Mixed-Behavior Policy Optimization.}
The reconstructed group contains $n-1$ trajectories sampled from
$\pi_{\mathrm{old}}$ and one trajectory sampled from the joint behavior
policy. We optimize the student using the following clipped surrogate
objective, averaged over training problems:
\begin{equation}
\begin{aligned}
\mathcal{J}_{\mathrm{PICO}(\theta)}
=
\mathbb{E}_{x\sim\mathcal{D}}
\Bigg[
\frac{1}{\mathcal{Z}}
&\Bigg(
\underbrace{
\sum_{i=1}^{n-1}
\sum_{t=1}^{|y^{(i)}|}
\min\!\left(
r_{i,t}(\theta)A^{(i)},
\operatorname{clip}\!\left(
r_{i,t}(\theta),1-\epsilon,1+\epsilon
\right)A^{(i)}
\right)
}_{\text{student trajectories}}
\\
&\quad+
\underbrace{
\sum_{t=1}^{|y^{\mathrm J}|}
\min\!\left(
\widehat r_{\mathrm J,t}(\theta)A^{\mathrm J},
\operatorname{clip}\!\left(
\widehat r_{\mathrm J,t}(\theta),1-\epsilon,1+\epsilon
\right)A^{\mathrm J}
\right)
}_{\text{joint trajectory}}
\Bigg)
\Bigg],
\end{aligned}
\label{eq:hype-objective}
\end{equation}
where
$\mathcal{Z}=\sum_{i=1}^{n-1}|y^{(i)}|+|y^{\mathrm J}|$ is the total number
of response tokens in the reconstructed group and $\widehat r_{\mathrm J,t}(\theta)$ is the off-policy
importance ratio for the joint trajectory
\begin{equation}
\widehat r_{\mathrm J,t}(\theta)
=
\frac{
\pi_\theta\!\left(y_t^{\mathrm J}\mid x,y_{<t}^{\mathrm J}\right)
}{
\mu_t\!\left(y_t^{\mathrm J}\right)
}
=
Z_t
\frac{
\pi_\theta\!\left(y_t^{\mathrm J}\mid x,y_{<t}^{\mathrm J}\right)
}{
\pi_{\mathrm{old}}\!\left(y_t^{\mathrm J}\mid x,y_{<t}^{\mathrm J}\right)
}.
\label{eq:joint-importance-ratio}
\end{equation}
At a fallback step, $\mu_t=p_t$ and $Z_t=1$, so
$\widehat r_{\mathrm J,t}(\theta)$ reduces to the standard policy ratio. This
importance ratio accounts for the token-level behavior-policy mismatch
introduced by joint decoding.

\section{Experiments}
\begin{figure*}[t]
    \centering
    \includegraphics[
        width=\textwidth,
        trim={0 10pt 0 0},
        clip
    ]{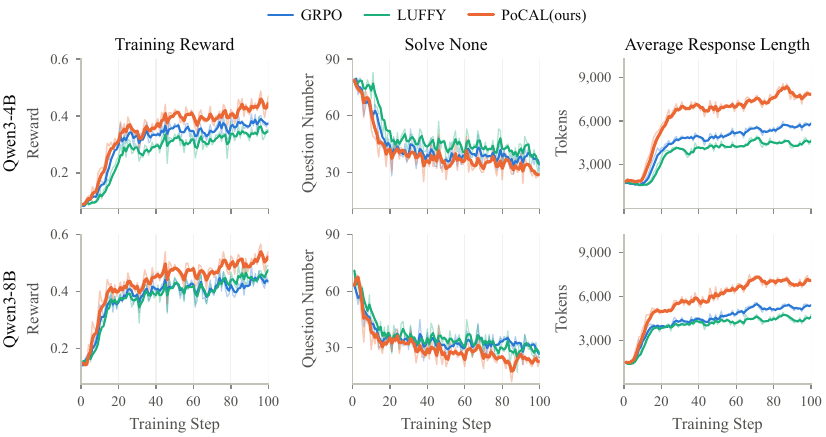}
    \caption{
    Training dynamics of Qwen3-4B-Base and Qwen3-8B-Base.
    We report the training reward, solve-none question number, and average response
    length throughout training.
    }
    \label{fig:train-dynamics}
\end{figure*}

\subsection{Experimental Setup}
\label{subsec:experimental-setup}
\textbf{Models and Datasets.} In our main experiments, we consider three student policies with different model scales and instruction-tuning configurations: Qwen3-4B-Base, Qwen3-8B-Base, and Qwen3-4B-Instruct\citep{yang2025qwen3}. We use Qwen3-4B-Instruct specifically for the agentic-task experiments. Across all settings, Qwen3-4B-Instruct, conditioned on expert solutions, serves as the teacher model. For mathematical and general reasoning tasks, we use DAPO-Math-17K\citep{yu2026dapo} as the training set, which contains a large collection of challenging problems paired with high-quality expert demonstrations. For agentic tasks, we conduct experiments on ALFWorld and ScienceWorld, two widely used interactive reasoning benchmarks.


\textbf{Baselines.} We compare \Mymethod{} against GRPO and three representative off-policy methods , as well as a self-distillation baseline. Specifically, LUFFY\citep{yan2025luffy} combines on-policy rollouts with off-policy demonstrations through mixed-policy GRPO and regularized policy shaping; Prefix-RFT\citep{huang2026prefixrft} constructs hybrid trajectories by pairing prefixes sampled from expert demonstrations with policy-generated continuations; and ReLIFT\citep{ma2026relift} interleaves reinforcement learning with targeted online supervised fine-tuning on challenging problems. We additionally compare with RLRT\citep{kim2026rlrt}, a self-distillation method that uses reversed teacher signals to emphasize student-preferred tokens in verified successful trajectories.

\textbf{Training Details. } For a fair comparison, we use the same training configuration for all RL methods. We remove the KL regularization term by setting $\beta=0$ and omit standard-deviation normalization in the advantage computation to avoid implicitly reweighting problems according to their within-group reward variance. We fix the total rollout group size at 8 for every query. For \Mymethod{} and all off-policy baselines, each mixed group consists of 7 on-policy trajectories and 1 off-policy trajectory. We use a global training batch size of 128, a PPO mini-batch size of 64, and a maximum response length of $20{,}480$ tokens. Following DAPO, we adopt asymmetric clipping with $\epsilon_{\mathrm{high}}=0.28$ and $\epsilon_{\mathrm{low}}=0.20$. The complete training configuration is provided in Appendix~\ref{training-details}.

\textbf{Evaluations.} Following prior work, we evaluate primarily on four widely used and challenging mathematical reasoning benchmarks: AIME 2025, AIME 2026, HMMT 2026, and MATH-500. During evaluation, we set the sampling temperature to $0.6$ and the maximum response length to $20{,}480$ tokens. To reduce the effect of sampling variance, we independently sample 16 responses for each problem and report the average accuracy across these samples (avg@16). To assess out-of-distribution generalization beyond the mathematics-focused training domain, we additionally evaluate on GPQA-Diamond and Humanity's Last Exam (HLE). Appendix~\ref{evaluation-benchmarks} provides additional information about the evaluation benchmarks.

\begin{table*}[t]
\centering

\caption{
Overall performance on in-domain and out-of-domain benchmarks, reported as avg@16. \textbf{Bold} and \underline{underlined} values indicate the best and second-best results, respectively. $\triangle$ denotes the absolute performance gain of our method over GRPO. Pass@$k$ results are reported in Table~\ref{tab:pass16-results}.}
\label{tab:overall-performance}

\footnotesize
\setlength{\tabcolsep}{6pt}

\begin{tabularx}{\textwidth}{
    @{}
    l
    *{7}{>{\centering\arraybackslash}X}
    @{}
}
\toprule

\multirow{2}{*}{\textbf{Method}}
& \multicolumn{4}{c}{\textbf{In-domain (\%)}}
& \multicolumn{2}{c}{\textbf{Out-of-domain (\%)}}
& \multirow{2}{*}{\textbf{Avg.}} \\

\cmidrule(lr){2-5}
\cmidrule(lr){6-7}

& \textbf{AIME25}
& \textbf{AIME26}
& \textbf{HMMT26}
& \textbf{MATH500}
& \textbf{GPQA-D}
& \textbf{HLE}
& \\

\midrule
\multicolumn{8}{c}{\textbf{Qwen3-4B-Base}} \\
\midrule
\textbf{Base}& 6.9 & 8.1 & 3.6 & 67.8 & 23.5 & 3.0 & 18.8  \\
\textbf{GRPO}& 21.7 & 14.4 & 10.0 & \underline{84.6} & 34.3 & 2.2 & 27.9 \\
\textbf{LUFFY}& \underline{22.1} & 18.1 & 14.0 & 84.2 & \underline{36.6} & 3.8 & \underline{29.8} \\
\textbf{Prefix-RFT}& 15.4 & 12.1 & 7.9 & 79.8 & 34.3 & 3.4 & 25.5 \\
\textbf{ReLIFT}& 10.6 & 10.2 & 6.8  & 76.2 & 31.5 & 3.8 & 23.2 \\
\textbf{RLRT}& 18.5 & \underline{19.8}  & \textbf{15.9} & 83.8 & 34.7  & \underline{4.2} & 29.5 \\
\rowcolor[HTML]{E8F3F1}
\textbf{\Mymethod{}}& \textbf{23.1} & \textbf{22.3}  & \textbf{15.9}  & \textbf{87.6} & \textbf{37.7} & \textbf{5.2} & \textbf{32.0}  \\

\rowcolor[HTML]{E8F3F1}
$\triangle$& +1.4 & +7.9 & +5.9 & +3.0 & +3.4 & +3.0 & +4.1 \\

\midrule

\multicolumn{8}{c}{\textbf{Qwen3-8B-Base}} \\

\midrule

\textbf{Base}& 10.2 &7.5  &4.4  &74.8  &30.1  &4.4  & 21.9 \\
\textbf{GRPO}& 22.7 & 18.9 &14.6  &86.4  & 44.4 &3.8  & 31.8 \\
\textbf{LUFFY}& \underline{24.4} & \underline{21.9} &14.9  & \underline{88.2}  & \underline{46.7}  & \underline{5.2}  & \underline{33.6} \\
\textbf{Prefix-RFT}& 23.9 & 21.0 &12.1  &85.8  &42.6  &3.4  & 31.5 \\
\textbf{ReLIFT}& 21.7  & 21.0 & 14.8 & 87.8 & 41.1 & 3.8 & 31.7 \\
\textbf{RLRT}& 18.8 &21.9  & \underline{15.9}  &84.4  & 45.7 & 4.6 & 31.9 \\
\rowcolor[HTML]{E8F3F1}
\textbf{\Mymethod{}}& \textbf{25.6} & \textbf{23.3}  & \textbf{18.9}  & \textbf{90.4}  & \textbf{48.8}  & \textbf{5.6}  & \textbf{35.4}  \\

\rowcolor[HTML]{E8F3F1}
$\triangle$& +2.9 & +4.4 & +4.3 & +4.0 & +4.4 & +1.8 & +3.6 \\

\bottomrule
\end{tabularx}

\end{table*}

\subsection{Main Results}
\textbf{Performance Comparison.}
As shown in Table~\ref{tab:overall-performance}, \Mymethod{} consistently achieves the strongest overall performance across both model scales. It improves the average score over GRPO by $4.1$ percentage points with Qwen3-4B-Base and by $3.6$ points with Qwen3-8B-Base, while outperforming the strongest competing method by $2.2$ and $1.8$ points, respectively. The gains are particularly pronounced on challenging benchmarks such as AIME26 and HMMT26, supporting our claim that \Mymethod{} enables effective learning from difficult problems that frequently produce all-incorrect rollout groups. Moreover, the consistent improvements on GPQA-Diamond and HLE demonstrate that the learned reasoning capability generalizes beyond the mathematics-focused training distribution. Since the teacher is unavailable during evaluation, these gains reflect improvements internalized by the student rather than reliance on privileged guidance.

\textbf{Training Dynamics.}
Figure~\ref{fig:train-dynamics} shows that \Mymethod{} maintains higher training rewards and consistently fewer solve-none questions than GRPO and LUFFY across both model scales. While the baselines tend to plateau after the initial stage, \Mymethod{} continues to improve, indicating that failure-triggered joint decoding converts otherwise uninformative all-incorrect groups into useful learning signals. \Mymethod{} also produces longer responses, which, together with the higher verified rewards and lower solve-none counts, suggests more sustained reasoning on difficult problems. Overall, these results show that \Mymethod{} improves difficult-problem exploration without compromising training stability. Additional analyses in Appendix~\ref{app:traing-curves} and Fig.\ref{fig:rescue-dynamics} show that joint decoding consistently rescues questions for which all student self-sampled trajectories fail, thereby restoring effective learning signals for otherwise uninformative groups.

\subsection{Agentic Tasks}
We conduct the agentic-task experiments using Qwen3-4B-Instruct as both the student policy and the teacher providing token-level guidance. As shown in Table~\ref{tab:agentic-results}, \Mymethod{} consistently outperforms all baselines across ALFWorld and ScienceWorld, achieving the highest average success rate of $74.3\%$. Compared with GRPO, \Mymethod{} improves the average performance by $6.8$ percentage points, demonstrating its effectiveness in interactive environments characterized by long-horizon decision-making and sparse rewards. It also surpasses the off-policy RL method LUFFY and the prefix-guided method Prefix-RFT by $4.9$ and $5.2$ points on average, respectively, further validating the effectiveness of token-level joint guidance for agentic reasoning.

\begin{table}[t]
\centering

\caption{
Performance on agentic reasoning benchmarks based on
Qwen3-4B-Instruct.
\textbf{Bold} and \underline{underlined} values indicate the best
and second-best results, respectively.
$\triangle$ denotes the absolute performance gains, in percentage
points, of our method over GRPO.
}
\label{tab:agentic-results}

\footnotesize
\setlength{\tabcolsep}{10pt}

\begin{tabular}{@{}lcccc@{}}
\toprule

\multirow{2}{*}{\textbf{Method}}
& \multicolumn{2}{c}{\textbf{ALFWorld (\%)}}
& \multirow{2}{*}{\textbf{ScienceWorld (\%)}} 
& \multirow{2}{*}{\textbf{Avg.}}\\

\cmidrule(lr){2-3}

& \textbf{Seen}
& \textbf{Unseen}
& \\

\midrule

\multicolumn{4}{c}{\textbf{Qwen3-4B-Instruct}} \\

\midrule

\textbf{GRPO} & 80.7 & 71.6 & 50.3 & 67.5 \\
\textbf{LUFFY} & 81.4 & \underline{74.6} & 52.1 & \underline{69.4} \\
\textbf{Prefix-RFT} & \underline{82.1} & 72.4 & \underline{52.9} & 69.1 \\

\rowcolor[HTML]{E8F3F1}
\textbf{\Mymethod{}} & \textbf{88.6} & \textbf{79.1} & \textbf{55.3} & \textbf{74.3} \\

\rowcolor[HTML]{E8F3F1}
$\triangle$ & $+7.9$ & $+7.5$ & $+5.0$ & $+6.8$ \\

\bottomrule
\end{tabular}

\end{table}

\begin{table*}[t]
\centering

\caption{
Ablation study of \Mymethod{} on four mathematical reasoning
benchmarks using Qwen3-4B-Base.
``w/o Joint Decoding'' directly generates responses from the teacher
policy.
The fixed-$k$ variants replace the dynamic top-$k$ mechanism with a
fixed teacher candidate-set size.
``w/o IS Ratio'' removes the importance-sampling ratio from policy
optimization.
\textbf{Bold} and \underline{underlined} values indicate the best
and second-best results, respectively.
}
\label{tab:ablation-study}

\footnotesize
\setlength{\tabcolsep}{9pt}

\begin{tabular}{@{}lccccc@{}}
\toprule

\textbf{Method / Variant}
& \textbf{AIME25}
& \textbf{AIME26}
& \textbf{HMMT26}
& \textbf{MATH500}
& \textbf{Avg.} \\

\midrule

\multicolumn{6}{c}{\textbf{Qwen3-4B-Base}} \\

\midrule

\textbf{GRPO} & 21.7 & 14.4 & 10.0 & 84.6 & 32.7 \\

\rowcolor[HTML]{E8F3F1}
\textbf{\Mymethod{}} & \underline{23.1}  & \textbf{22.3} & \underline{15.9} & \textbf{87.6} & \textbf{37.2} \\

\midrule

\textbf{w/o Joint Decoding} & 20.6 & 16.9 & 14.1 & 83.6 & 33.8 \\

\addlinespace[2pt]
\multicolumn{6}{@{}l}{\textbf{w/o Dynamic Top-$k$}} \\

\hspace{1em}Fixed $k=5$ & \textbf{23.3} & 16.9 & 12.7 & 84.2  & 34.3 \\

\hspace{1em}Fixed $k=10$ & 23.1 & \underline{22.1} & 13.4 & \underline{84.8}  & \underline{35.9} \\

\hspace{1em}Fixed $k=20$ & 21.7  & 18.3  & \textbf{17.1} & 84.0  & 35.3 \\

\addlinespace[2pt]
\textbf{w/o Off-Policy IS Ratio} & 21.7 & 19.4 & 15.2 & \underline{84.8} & 35.3 \\

\bottomrule
\end{tabular}

\end{table*}

\subsection{Ablation Study}
We analyze the contribution of each component of \Mymethod{}. As shown in Table~\ref{tab:ablation-study}, the results highlight the importance of maintaining compatibility between guided trajectories and the student policy. Directly sampling from the teacher without joint decoding introduces the largest behavior mismatch and yields the weakest ablation result, reducing the average score from $37.2\%$ to $33.8\%$. The fixed-$k$ variants further reveal a trade-off between teacher guidance and student compatibility. Increasing $k$ retains more student-compatible tokens but weakens the selectivity of teacher guidance, potentially reducing rollout accuracy. As long as the candidate set remains sufficiently accurate, increasing $k$ improves learning, as evidenced by the gain from $k=5$ to $k=10$. However, further increasing $k$ to $20$ weakens teacher guidance and slightly reduces performance, making $k=10$ the best fixed setting. In contrast, our dynamic top-$k$ strategy adapts the candidate set at each decoding step according to the retained student probability mass, achieving a better balance and the highest average score of $37.2\%$. Finally, removing the off-policy importance-sampling ratio reduces performance to $35.3\%$, confirming the importance of explicitly accounting for the remaining behavior-policy mismatch during optimization.

\section{Related Work}
\subsection{On- and Off-Policy Learning}
Recent studies show that on-policy optimization is fundamentally bounded by the current policy’s exploration capacity, often causing learning to stall on challenging problems \citep{yue2025does}. To overcome this exploration bottleneck, a growing line of work incorporates off-policy expert trajectories into RLVR training\citep{dong2026rlplus,liu2025superrl}. LUFFY\citep{yan2025luffy} combines expert demonstrations with on-policy rollouts through mixed-policy optimization, while ReLIFT\citep{ma2026relift} interleaves reinforcement learning with supervised fine-tuning on high-quality solutions to challenging problems. SRFT and Chord further unifies supervised learning from expert demonstrations and reinforcement learning from self-generated rollouts within a single training stage\citep{fu2026srft,zhang2026chord}. 
However, because expert trajectories may lie in low-probability regions under the current student policy, directly optimizing on them can induce substantial distribution mismatch, potentially leading to excessive policy drift and unstable optimization \citep{agrawal2026offgrpo,setlur2026reuse,zheng2026prosperity,ICLR2026_donotlow}. In contrast, \Mymethod{} uses expert trajectories to guide distribution-constrained joint decoding, producing successful rollouts that remain compatible with the student policy.

\subsection{Expert-Guided Exploration and Decoding}
Rather than directly optimizing on off-policy demonstrations, another line of work uses expert information to guide rollout generation\citep{jiang2026selectiveguidance,zhang2025critiquegrpo,liu2026em2po}. Prefix-RFT\citep{huang2026prefixrft} conditions rollouts on prefixes sampled from expert demonstrations, while BREAD and StepHint use expert-derived prefixes or intermediate hints to steer exploration toward promising reasoning states\citep{zhang2026bread,zhang2026stephint}. OC-GRPO\citep{agrawal2026offgrpo} further applies importance sampling ratio to account for the mismatch between guided rollouts and the original unguided policy. Although these methods facilitate successful exploration, injecting guidance into the rollout context may teach the student to continue from expert-induced states without learning how to reach those states independently from the original prompt\citep{qu2026pope,ross2011dagger,zhang2026scaf}. A related line of work combines or contrasts the token distributions of models with different capabilities to steer generation toward higher-quality outputs\citep{liu2021dexperts,li2023contrastivedecoding,dekoninck2024controlledgeneration}.
The effectiveness of these methods motivates us to incorporate teacher guidance at the token level. Accordingly, \Mymethod{} employs distribution-constrained joint decoding to construct policy-compatible trajectories.

\section{Conclusion}
\label{sec:conclusion}

In this work, we introduced \Mymethod{}, a policy-compatible reinforcement learning framework for overcoming the exploration bottleneck of RLVR on challenging problems. \Mymethod{} leverages expert data to generate verified trajectories while controlling their deviation from the student policy, and selectively incorporates these trajectories into all-failure groups to restore informative learning signals. Extensive experiments across mathematical reasoning, general reasoning, and agentic tasks demonstrate that \Mymethod{} consistently outperforms standard RLVR and representative expert-guided baselines across different model scales, without requiring teacher guidance at inference time. Ablation studies further validate the effectiveness of each component. Overall, our results highlight policy compatibility as a key principle for effectively incorporating expert knowledge into RLVR, providing a practical approach to improving models beyond the limits of unaided on-policy exploration.

\bibliographystyle{iclr2027_conference}
\bibliography{reference}

\newpage

\subsection*{AI use statement}
In this work, we used generative AI tools to assist in the writing of proofs and with translation. We did not use generative AI tools to formulate mathematical claims, provide critical ingredients for proving mathematical claims, propose or refine hypotheses, design or provide feedback on research methodology or experiments, implement methods, support qualitative or thematic data analysis, or interpret results. Generating synthetic datasets, developing theoretical models or conceptual frameworks, and cleaning or reformatting datasets were not applicable to this work. Additionally, we used generative AI tools to create or modify scientific figures or images, draft parts of the manuscript, support brainstorming, and source or search for information. We have reviewed all AI-assisted work. In particular, the authors manually verified every step of the AI-assisted proofs; checked translations against the original text to preserve their technical meaning; inspected AI-assisted figures and images for consistency with the underlying results; and carefully reviewed and revised all AI-assisted manuscript text. Ideas and information obtained through AI-assisted brainstorming or search were treated only as preliminary suggestions and were independently validated through manual literature review and consultation of the original sources. We take responsibility for the final content of this work, including text, claims, or artifacts produced with the aid of generative AI.

\subsection*{Ethics statement}
This work does not involve human subjects, personally identifiable or sensitive information, or the collection and release of new datasets. All experiments are conducted using publicly available models, datasets, and benchmarks for research purposes. Our method focuses on improving reinforcement learning with verifiable rewards and is not designed for high-risk or directly harmful applications.

\subsection*{Reproducibility Statement}

To support reproducibility, Section~\ref{method} provides a complete specification of the proposed method, while Section~\ref{subsec:experimental-setup} summarizes the main experimental setup. Appendix~\ref{appendix-experiments} provides further details on the training procedure, benchmark construction, data processing, and evaluation configurations, with the complete training hyperparameters reported in Tables~\ref{tab:implementation-configuration} and \ref{tab:agentic-implementation-configuration}. All prompts used for training and evaluation are provided in Appendix~\ref{appendix-prompts}.
\appendix

\section{Methodological Supplements}
\label{appendix-method}
\subsection{Theoretical Proofs}
\label{theoretical-proofs}

This section provides the derivations underlying the joint-decoding policy
and makes explicit the guarantees induced by the retained-mass constraint.
We assume that the student and teacher share a tokenizer and vocabulary
$\mathcal V$. Since $p_t$ is obtained from a softmax distribution, we have
$p_t(v)>0$ for every $v\in\mathcal V$. We also assume a fixed tie-breaking
rule for top-$k$ selection, so the selected sets are nested as $k$ increases.
The retained-mass threshold satisfies $Z_{\min}\in(0,1]$. All token-level
statements below are conditioned on a fixed query, privileged guidance, and
generated prefix; we suppress this conditioning when it is clear from
context.

\subsubsection{Optimization Problem Formulation and Closed-Form Solution}
\label{appendix-joint-optimization}

For a nonempty consensus set $F_t$, Eq.~\eqref{eq:joint-policy-optimization}
can be written as
\begin{equation}
\min_{\mu\in\Delta(F_t)}
\sum_{v\in F_t}\mu(v)
\log\frac{\mu(v)}{p_t(v)}
-\lambda\sum_{v\in F_t}\mu(v)q_t(v).
\label{eq:appendix-joint-objective}
\end{equation}
Introducing a Lagrange multiplier $\eta$ for the normalization constraint
gives
\begin{equation}
\mathcal L(\mu,\eta)
=
\sum_{v\in F_t}\mu(v)
\log\frac{\mu(v)}{p_t(v)}
-\lambda\sum_{v\in F_t}\mu(v)q_t(v)
+\eta\left(\sum_{v\in F_t}\mu(v)-1\right).
\end{equation}
At the optimum, stationarity with respect to each $\mu(v)$ yields
\begin{equation}
\log\frac{\mu(v)}{p_t(v)}+1-\lambda q_t(v)+\eta=0,
\end{equation}
and hence
\begin{equation}
\mu(v)=p_t(v)\exp\!\left(\lambda q_t(v)-1-\eta\right).
\end{equation}
Normalizing over $F_t$ gives
\begin{equation}
\mu_t^\lambda(v)
=
\frac{p_t(v)\exp(\lambda q_t(v))\mathbf 1\{v\in F_t\}}
{\sum_{u\in F_t}p_t(u)\exp(\lambda q_t(u))},
\end{equation}
which is Eq.~\eqref{eq:soft-joint-policy}. The KL term is strictly convex in
$\mu$ on $\Delta(F_t)$, while the teacher-confidence term is linear;
therefore, this stationary point is the unique global minimizer.

\subsubsection{Proof of the KL Decomposition}
\label{appendix-kl-proof}

Let $Z_t=\sum_{v\in F_t}p_t(v)$ and define the student distribution
conditioned on the consensus support as
\begin{equation}
\mu_t^0(v)
=
\frac{p_t(v)\mathbf 1\{v\in F_t\}}{Z_t}.
\end{equation}
For any distribution $\mu$ supported on $F_t$,
\begin{align}
D_{\mathrm{KL}}(\mu\Vert p_t)
&=
\sum_{v\in F_t}\mu(v)
\log\frac{\mu(v)}{p_t(v)}
\nonumber\\
&=
\sum_{v\in F_t}\mu(v)
\log\frac{\mu(v)}{p_t(v)/Z_t}
-\log Z_t
\nonumber\\
&=
D_{\mathrm{KL}}(\mu\Vert\mu_t^0)-\log Z_t.
\label{eq:appendix-general-kl-decomposition}
\end{align}
Substituting $\mu=\mu_t^\lambda$ proves
Proposition~\ref{prop:kl-decomposition}. Since KL divergence is
nonnegative, the minimum divergence over all distributions supported on
$F_t$ is $-\log Z_t$, attained uniquely by $\mu_t^0$. Thus, once the support
has been fixed, soft teacher reweighting cannot decrease the distribution
shift from the student; for every $\lambda>0$, it increases the shift
strictly unless $q_t$ is constant on $F_t$.

This last statement can also be seen directly. Let
$m_t(\lambda)=\mathbb E_{v\sim\mu_t^\lambda}[q_t(v)]$. Standard properties
of exponential families give
\begin{equation}
\frac{d}{d\lambda}m_t(\lambda)
=
\operatorname{Var}_{v\sim\mu_t^\lambda}[q_t(v)]\geq 0,
\end{equation}
and
\begin{equation}
\frac{d}{d\lambda}
D_{\mathrm{KL}}(\mu_t^\lambda\Vert\mu_t^0)
=
\lambda\operatorname{Var}_{v\sim\mu_t^\lambda}[q_t(v)]\geq0.
\label{eq:appendix-soft-shift-monotonicity}
\end{equation}
Consequently, if $q_t$ is nonconstant on $F_t$, the expected teacher
confidence is strictly increasing on $[0,\infty)$, while the additional KL
cost is strictly increasing on $(0,\infty)$ and has its unique minimum at
$\lambda=0$. Our hard policy selects this minimum-shift endpoint and uses the
teacher only to determine the support.

\paragraph{Trajectory-level consequence.}
The per-token retained-mass constraint also controls the divergence of the
complete autoregressive behavior policy. This is formalized below.

\begin{proposition}[Trajectory-level KL bound]
\label{prop:trajectory-kl-bound}
Let $\mu^{\mathrm J}(y\mid x,g)$ be the autoregressive joint-decoding policy
and let $\pi_{\mathrm{old}}(y\mid x)$ be the rollout-time student policy. If
every guided step satisfies $Z_t\geq Z_{\min}$, every infeasible step falls
back to $p_t$, and generation has horizon $T_{\max}$, then
\begin{equation}
D_{\mathrm{KL}}
\!\left(
\mu^{\mathrm J}(\cdot\mid x,g)
\Vert
\pi_{\mathrm{old}}(\cdot\mid x)
\right)
\leq
T_{\max}(-\log Z_{\min}).
\label{eq:trajectory-kl-bound}
\end{equation}
Here, variable-length responses are represented on the fixed horizon by
treating the end-of-sequence state as absorbing and setting the two policies
equal after termination. More precisely, the divergence equals the expected
cumulative per-step cost
$\mathbb E_{y\sim\mu^{\mathrm J}}[\sum_{t=1}^{T_{\max}}-\log Z_t]$, where
$Z_t=1$ at fallback and post-termination steps.
\end{proposition}

\begin{proof}
By the chain rule for KL divergence between autoregressive distributions,
\begin{align}
D_{\mathrm{KL}}(\mu^{\mathrm J}\Vert\pi_{\mathrm{old}})
&=
\mathbb E_{y\sim\mu^{\mathrm J}}
\left[
\sum_{t=1}^{T_{\max}}
D_{\mathrm{KL}}
\!\left(
\mu_t(\cdot\mid y_{<t})
\Vert
p_t(\cdot\mid y_{<t})
\right)
\right]
\nonumber\\
&=
\mathbb E_{y\sim\mu^{\mathrm J}}
\left[
\sum_{t=1}^{T_{\max}}-\log Z_t
\right].
\label{eq:trajectory-kl-chain-rule}
\end{align}
The second equality follows from
Eq.~\eqref{eq:appendix-general-kl-decomposition} at $\mu=\mu_t^0$.
A fallback step has $\mu_t=p_t$ and therefore contributes zero. At every
guided step, $Z_t\geq Z_{\min}$ implies
$-\log Z_t\leq-\log Z_{\min}$. Summing over the horizon gives
Eq.~\eqref{eq:trajectory-kl-bound}.
\end{proof}

\subsubsection{Importance-Sampling Correction}
\label{appendix-importance-sampling}

At a fixed prefix $h_t=(x,y_{<t})$, for any sampled token $v\in F_t(h_t)$,
the hard policy satisfies
\begin{equation}
\frac{\pi_\theta(v)}{\mu_t^{\mathrm H}(v)}
=
Z_t\frac{\pi_\theta(v)}{p_t(v)},
\end{equation}
which gives Eq.~\eqref{eq:joint-importance-ratio}. The factor $Z_t$ exactly
corrects the renormalization of the student probabilities on $F_t$. More
formally, for any bounded token-level quantity $f$,
\begin{equation}
\mathbb E_{v\sim\mu_t^{\mathrm H}}
\left[
\frac{\pi_\theta(v)}{\mu_t^{\mathrm H}(v)}f(v)
\right]
=
\sum_{v\in F_t}\pi_\theta(v)f(v).
\label{eq:support-restricted-is}
\end{equation}
Thus, conditional on the prefix, the action-level correction is exact on the
retained support. As with any hard-support behavior policy, it cannot recover
contributions from tokens outside $F_t$. At data collection, when
$\pi_\theta=p_t$, the magnitude of the omitted contribution is bounded by
\begin{equation}
\left|
\mathbb E_{v\sim p_t}[f(v)]
-
\mathbb E_{v\sim\mu_t^{\mathrm H}}[Z_t f(v)]
\right|
\leq
\lVert f\rVert_\infty(1-Z_t)
\leq
\lVert f\rVert_\infty(1-Z_{\min}).
\label{eq:support-bias-bound}
\end{equation}
For the rollout-time student, the retained-mass constraint therefore controls
both the KL shift and the residual support truncation. Moreover, at data
collection the IS factor is the token-independent constant $Z_t$, avoiding
the additional within-step weight variance that would arise from soft
teacher reweighting.

Equation~\eqref{eq:support-restricted-is} is a conditional, action-level
identity; it does not convert prefixes sampled under $\mu^{\mathrm J}$ into
prefixes sampled under $\pi_\theta$. A fully trajectory-corrected estimator
would require cumulative importance ratios and target-policy absolute
continuity, $\pi_\theta\ll\mu^{\mathrm J}$, over complete trajectories. This
condition fails whenever the hard support restriction removes a token to
which $\pi_\theta$ assigns positive probability. In addition, PPO clipping
deliberately introduces bias into the surrogate. Accordingly,
Eq.~\eqref{eq:hype-objective} should be understood as a mixed-behavior
surrogate with exact per-token probability ratios on sampled actions, rather
than an unbiased estimator of the original on-policy RLVR objective.

\subsection{Algorithmic Procedure}
\label{algorithmic-procedure}

Algorithm~\ref{alg:picorl} gives the complete failure-triggered training
procedure. The teacher is frozen throughout training, and $\pi_{\mathrm{old}}$
is held fixed during rollout collection for a policy update. The privileged
guidance $g(x)$ is supplied only to the teacher; the student always evaluates
the original query and the jointly generated response prefix.

\begin{algorithm}[t]
\caption{\Mymethod{}: Policy-Compatible Reinforcement Learning Framework}
\label{alg:picorl}
\small
\begin{algorithmic}[1]

\Require Training set $\mathcal D=\{(x,g_x)\}$; student $\pi_\theta$;
teacher $\pi_{\mathrm T}$; verifier $R$; batch size $B$; group size $n$;
$K_{\mathrm S}$; $K_{\mathrm T}^{\max}$; $Z_{\min}$; $T_{\max}$;
$\epsilon_{\mathrm{low}}$; $\epsilon_{\mathrm{high}}$

\Function{JointDecode}{$x,g_x,\pi_{\mathrm{old}},\pi_{\mathrm T}$}
    \State $y\gets()$
    \While{$y$ is not terminated and $|y|<T_{\max}$}
        \State $p_t\gets\pi_{\mathrm{old}}(\cdot\mid x,y)$,
        \quad
        $q_t\gets\pi_{\mathrm T}(\cdot\mid x,g_x,y)$

        \State $F_t(k)\gets
        \operatorname{TopK}(q_t,k)
        \cap
        \operatorname{TopK}(p_t,K_{\mathrm S})$

        \State $Z_t(k)\gets
        \sum_{v\in F_t(k)}p_t(v)$,
        \quad $k=1,\ldots,K_{\mathrm T}^{\max}$

        \State $\mathcal K_t\gets
        \{k:Z_t(k)\geq Z_{\min}\}$

        \If{$\mathcal K_t\neq\varnothing$}
            \State $k_t^\star\gets\min\mathcal K_t$
            \Comment{dynamic top-$k$}

            \State $F_t\gets F_t(k_t^\star)$,
            \quad
            $Z_t\gets Z_t(k_t^\star)$

            \State $\mu_t(v)\gets
            \dfrac{
                p_t(v)\mathbf 1\{v\in F_t\}
            }{
                Z_t
            }$
            \Comment{preserve student preferences}
        \Else
            \State $\mu_t\gets p_t$
        \EndIf

        \State Sample $y_t\sim\mu_t$ and append $y_t$ to $y$
    \EndWhile
    \State \Return $y$ and $\{\mu_t\}_{t=1}^{|y|}$
\EndFunction

\State Define
$\ell(\rho,A)\gets
\min\!\left\{
\rho A,\,
\operatorname{clip}\!\left(
\rho,1-\epsilon_{\mathrm{low}},1+\epsilon_{\mathrm{high}}
\right)A
\right\}$

\ForAll{$\mathcal B\in
\operatorname{MiniBatches}(\mathcal D,B)$}

    \State $\pi_{\mathrm{old}}\gets\pi_\theta$

    \ForAll{$(x,g_x)\in\mathcal B$}
        \State Sample
        $\{y_x^{(i)}\}_{i=1}^{n}
        \sim\pi_{\mathrm{old}}(\cdot\mid x)$

        \State $r_x^{(i)}\gets R(x,y_x^{(i)})$,
        \quad
        $b_{x,i,t}\gets
        \pi_{\mathrm{old}}(\cdot\mid x,y_{x,<t}^{(i)})$

        \If{$\sum_{i=1}^{n}r_x^{(i)}=0$}
            \Comment{trigger on an all-failed group}

            \State $(y^{\mathrm J},\{\mu_t\})
            \gets
            \Call{JointDecode}
            {x,g_x,\pi_{\mathrm{old}},\pi_{\mathrm T}}$

            \If{$R(x,y^{\mathrm J})=1$}
                \State Replace
                $(y_x^{(n)},r_x^{(n)},b_{x,n,t})$
                with
                $(y^{\mathrm J},1,\mu_t)$
                \Comment{replace one failed rollout}
            \EndIf
        \EndIf

        \State $A_x^{(i)}\gets
        r_x^{(i)}
        -
        \dfrac{1}{n}\sum_{j=1}^{n}r_x^{(j)}$,
        \quad $i=1,\ldots,n$
    \EndFor

    \State $\rho_{x,i,t}(\theta)\gets
    \dfrac{
        \pi_\theta(y_{x,t}^{(i)}
        \mid x,y_{x,<t}^{(i)})
    }{
        b_{x,i,t}(y_{x,t}^{(i)})
    }$
    \Comment{token-level IS correction}

    \State $\mathcal Z_{\mathcal B}\gets
    \sum_{x\in\mathcal B}
    \sum_{i=1}^{n}|y_x^{(i)}|$

    \State Update $\theta$ by maximizing
    $\displaystyle
    \mathcal J_{\mathcal B}(\theta)
    =
    \frac{1}{\mathcal Z_{\mathcal B}}
    \sum_{x\in\mathcal B}
    \sum_{i=1}^{n}
    \sum_{t=1}^{|y_x^{(i)}|}
    \ell\!\left(
        \rho_{x,i,t}(\theta),
        A_x^{(i)}
    \right)$

\EndFor

\end{algorithmic}
\end{algorithm}

\begin{figure*}[t]
    \centering
    \includegraphics[
        width=0.8\textwidth
    ]{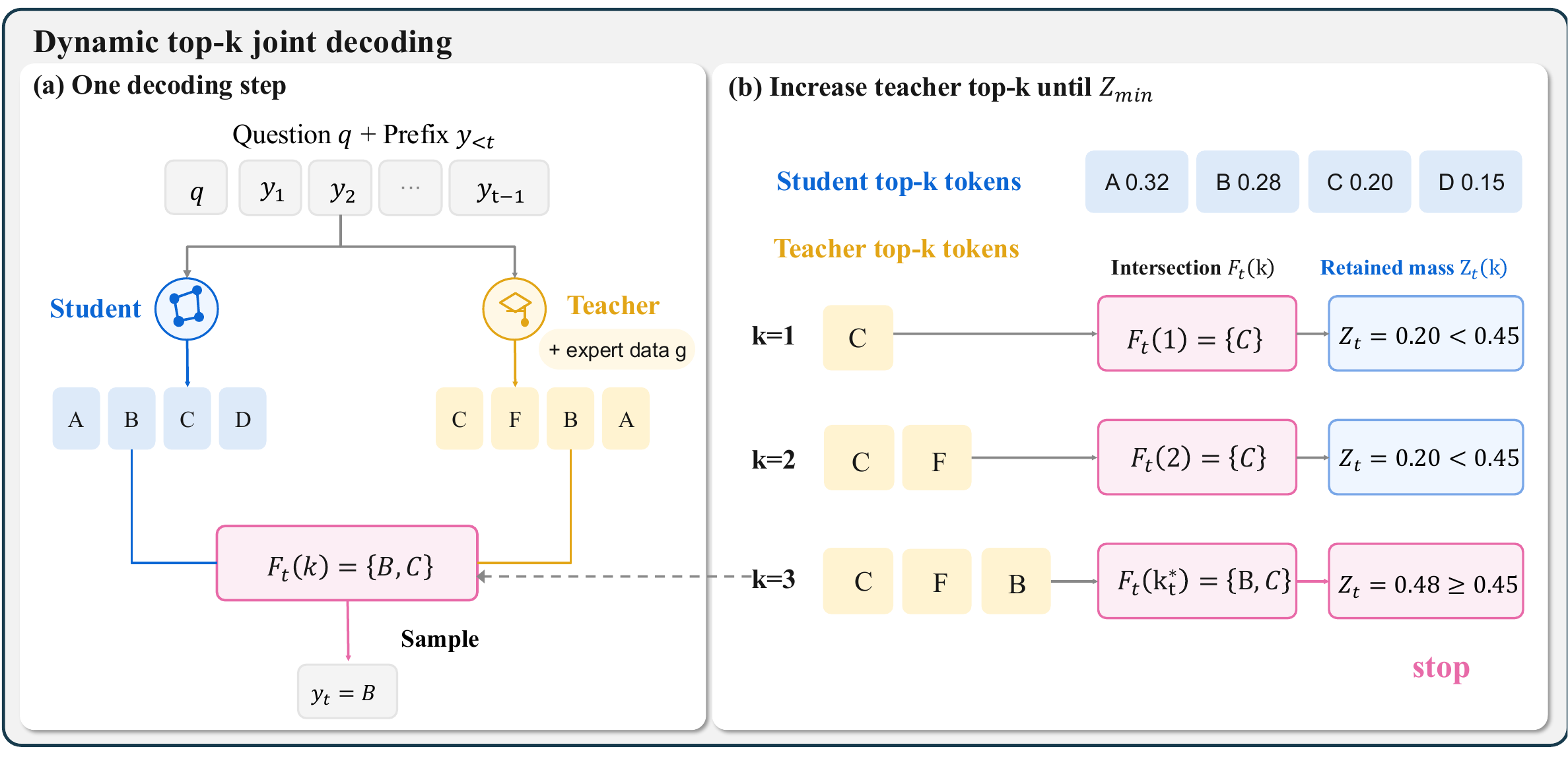}
    \caption{
    Illustration of dynamic top-$k$ joint decoding.
    At each decoding step, the student and teacher policies produce
    their respective top-$k$ token sets, whose intersection defines
    the candidate set $F_t(k)$.
    The teacher-side value of $k$ is progressively increased until
    the retained student probability mass $Z_t(k)$ reaches the
    prescribed threshold $Z_{\min}$.
    The next token is then sampled from the student distribution
    restricted to the resulting candidate set.
    }
    \label{fig:joint-decoding}
\end{figure*}

\section{Additional Experimental Results}
\label{appendix-results}

\subsection{Pass@$k$ Results}
\label{pass-at-k-results}

As shown in Table~\ref{tab:pass16-results}, \Mymethod{} substantially improves Pass@16 across both model scales. Compared with GRPO, it raises the average Pass@16 by $8.8$ percentage points on Qwen3-4B-Base and $10.8$ points on Qwen3-8B-Base. With the 4B model, \Mymethod{} achieves the best average result of $43.3\%$, outperforming the strongest baseline by $6.6$ points. With the 8B model, it reaches $45.5\%$ and matches the best baseline. These gains indicate that \Mymethod{} improves not only average single-sample accuracy but also the likelihood of discovering at least one correct solution under repeated sampling, supporting its effectiveness in expanding successful exploration on difficult problems.

\begin{table*}[t]
\centering

\caption{
Pass@16 performance on three mathematical reasoning benchmarks.
\textbf{Bold} and \underline{underlined} values indicate the best
and second-best results, respectively.
$\triangle$ denotes the absolute performance gains, in percentage
points, of our method over GRPO.
}
\label{tab:pass16-results}

\footnotesize
\setlength{\tabcolsep}{12pt}

\begin{tabular}{@{}l*{4}{c}@{}}
\toprule

\textbf{Method}
& \textbf{AIME25}
& \textbf{AIME26}
& \textbf{HMMT26}
& \textbf{Avg.} \\

\midrule
\multicolumn{5}{c}{\textbf{Qwen3-4B-Base}} \\
\midrule

\textbf{Base}       & 26.7 & 16.7 & 18.2 & 20.5 \\
\textbf{GRPO}       & 33.3 & 40.0 & 30.3 & 34.5 \\
\textbf{LUFFY}      & 40.0 & 40.0 & 27.3 & 35.8 \\
\textbf{Prefix-RFT} & 40.0 & 36.7 & 24.2 & 33.6 \\
\textbf{ReLIFT}     & 40.0 & 23.3 & 24.2 & 29.2 \\
\textbf{RLRT}       & 36.7 & 40.0 & 33.3 & 36.7 \\

\rowcolor[HTML]{E8F3F1}
\textbf{\Mymethod{}} & 46.7 & 50.0 & 33.3 & 43.3 \\

\rowcolor[HTML]{E8F3F1}
$\triangle$ & $+13.4$ & $+10.0$ & $+3.0$ & $+8.8$ \\

\midrule
\multicolumn{5}{c}{\textbf{Qwen3-8B-Base}} \\
\midrule

\textbf{Base}       & 33.3 & 23.3 & 18.2 & 24.9 \\
\textbf{GRPO}       & 36.7 & 40.0 & 27.3 & 34.7 \\
\textbf{LUFFY}      & 43.3 & 46.7 & 27.3 & 39.1 \\
\textbf{Prefix-RFT} & 46.7 & 40.0 & 30.3 & 39.0 \\
\textbf{ReLIFT}     & 50.0 & 46.7 & 30.3 & 42.3 \\
\textbf{RLRT}       & 50.0 & 53.3 & 33.3 & 45.5 \\

\rowcolor[HTML]{E8F3F1}
\textbf{\Mymethod{}} & 50.0 & 53.3 & 33.3 & 45.5 \\

\rowcolor[HTML]{E8F3F1}
$\triangle$ & $+13.3$ & $+13.3$ & $+6.0$ & $+10.8$ \\

\bottomrule
\end{tabular}

\end{table*}

\subsection{Training Curves}
\label{app:traing-curves}

\textbf{Question Rescue through Joint Decoding.}
Figure~\ref{fig:rescue-dynamics} directly illustrates how joint decoding recovers learning signals from questions that are entirely unsolved by student self-sampling. Across training, joint decoding rescues approximately $80\%$ and $76\%$ of these questions for Qwen3-4B and Qwen3-8B on mathematical reasoning tasks, respectively, substantially reducing the number of all-incorrect groups. It also rescues $36\%$ of the failed questions on ALFWorld, despite the greater difficulty of long-horizon interaction. These results confirm that joint decoding enables the policy to obtain successful trajectories on problems beyond the reach of its current self-exploration, thereby converting otherwise uninformative groups into effective training signals.

\textbf{Behavior-Policy Divergence.}
Figure~\ref{fig:kl-dynamics} shows that the retained student probability mass $Z_t$ remains consistently high throughout training, increasing from approximately $0.91$ to $0.94$. Correspondingly, the induced token-level KL divergence, $-\log Z_t$, decreases from roughly $0.10$ to $0.06$. This trend indicates that dynamic top-$k$ joint decoding preserves most of the student probability mass and becomes increasingly compatible with the evolving student policy, thereby providing teacher guidance without introducing a large behavior-policy shift.

\label{training-curves}
\begin{figure*}[t]
    \centering
    \includegraphics[
        width=0.92\textwidth
    ]{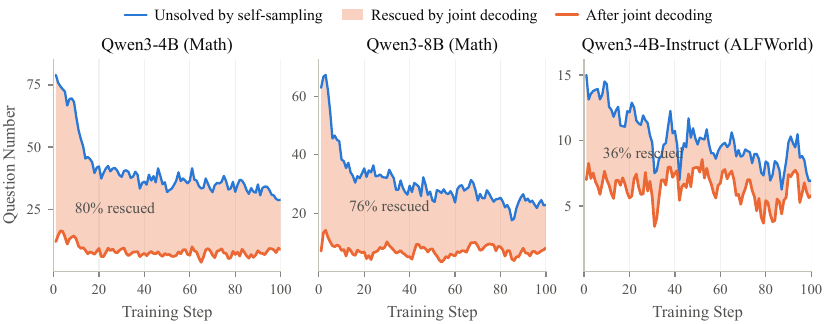}
    \caption{
    Dynamics of question rescue during training.
    The blue curves show the number of questions remaining unsolved
    through student self-sampling, while the orange curves show the
    number remaining unsolved after distribution-constrained joint
    decoding.
    The shaded regions represent questions successfully rescued by
    joint decoding.
    Results are reported for Qwen3-4B-Base and Qwen3-8B-Base on
    mathematical reasoning tasks, and Qwen3-4B-Instruct on ALFWorld.
    }
    \label{fig:rescue-dynamics}
\end{figure*}

\begin{figure*}[t]
    \centering
    \includegraphics[
        width=0.8\textwidth
    ]{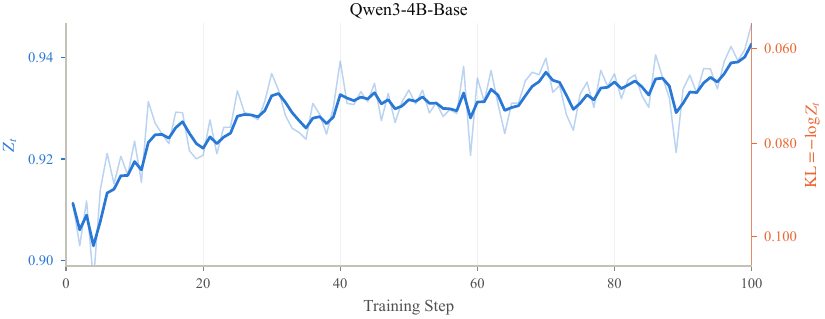}
    \caption{
    Evolution of the normalization term $Z_t$ and the corresponding
    KL divergence, $\mathrm{KL}=-\log Z_t$, during training with
    Qwen3-4B-Base.
    The light blue curve shows the raw measurements, while the dark
    blue curve shows the smoothed trend.
    }
    \label{fig:kl-dynamics}
\end{figure*}

\newpage

\newpage

\section{Experimental Details}
\label{appendix-experiments}
We build all methods on top of the \texttt{verl} framework. For mathematical reasoning, we use Qwen3-4B-Base and Qwen3-8B-Base as the student policies. For a fair comparison, all reasoning methods are trained on the same DAPO-Math-17K corpus~\citep{yu2026dapo} using a shared hyperparameter configuration. For the agentic experiments, we train Qwen3-4B-Instruct separately on ALFWorld and ScienceWorld. Across all settings, Qwen3-4B-Instruct serves as the teacher and conditions on expert solutions to provide token-level guidance. All experiments are conducted on four NVIDIA H20 GPUs. Training either Qwen3-4B-Base or Qwen3-8B-Base on DAPO-Math-17K takes approximately $1.5$--$2$ days, while training Qwen3-4B-Instruct for each agentic task takes approximately 1 day.

\subsection{Training Details}
\label{training-details}

\textbf{Training Efficiency.}
As shown in Table~\ref{tab:math-training-efficiency}, \Mymethod{} increases the average time per training step from $28.3$ to $43.6$ minutes due to the additional teacher-guided decoding. This additional computation, however, improves rollout utilization from $57.2\%$ to $91.9\%$, allowing substantially more sampled groups to provide effective learning signals. The utilization-per-time ratio also remains comparable to that of GRPO ($2.11$ vs.\ $2.01$ percentage points per minute). Thus, \Mymethod{} introduces a clear wall-clock overhead, but recovers considerably more useful training samples from each rollout batch.

\begin{table*}[t]
\centering
\caption{Implementation and hyperparameter configuration for mathematical reasoning experiments.}
\label{tab:implementation-configuration}
\small
\setlength{\tabcolsep}{8pt}
\renewcommand{\arraystretch}{1.08}

\begin{tabularx}{\textwidth}{
    @{}
    >{\raggedright\arraybackslash}p{0.38\textwidth}
    >{\raggedright\arraybackslash}X
    @{}
}
\toprule
\textbf{Hyperparameter} & \textbf{Value} \\
\midrule

\addlinespace[3pt]
\multicolumn{2}{@{}l}{\textit{Common (shared by all methods)}} \\[2pt]
\midrule
Optimizer                  & AdamW \\
Learning rate              & 1e-6 \\
Global training batch size & 128 \\
PPO mini-batch size        & 64 \\
Rollout group size         & 8 \\
Mixed-group composition    & 7 on-policy trajectories + 1 joint/off-policy trajectory \\
Maximum prompt length      & 2048 \\
Maximum response length    & 20480 \\
KL coefficient $\beta$     & 0 \\
Entropy coefficient        & 0 \\
PPO clipping $(\epsilon_{\mathrm{low}},\epsilon_{\mathrm{high}})$ & $(0.20, 0.28)$ \\
Advantage normalization    & disabled \\
Loss aggregation           & token-level \\
Temperature                & 1.0 \\
\addlinespace[3pt]
\midrule
\multicolumn{2}{@{}l}{\textit{\Mymethod{} configuration}} \\[2pt]
\midrule
Joint-decoding policy & Hard joint policy ($\lambda=0$) \\
Joint-decoding trigger & All rollouts fail \\
Joint-decoding response length & 10240 \\
Student cutoff $K_{\mathrm S}$ & 50 \\
Maximum teacher cutoff $K_{\mathrm T}^{\max}$ & 20 \\
Retained-mass threshold $Z_{\min}$ & 0.8 \\

\bottomrule
\end{tabularx}
\end{table*}

\begin{table*}[t]
\centering
\caption{Implementation and hyperparameter configuration for agentic experiments.}
\label{tab:agentic-implementation-configuration}
\small
\setlength{\tabcolsep}{8pt}
\renewcommand{\arraystretch}{1.08}

\begin{tabularx}{\textwidth}{
    @{}
    >{\raggedright\arraybackslash}p{0.38\textwidth}
    >{\raggedright\arraybackslash}X
    @{}
}
\toprule
\textbf{Hyperparameter} & \textbf{Value} \\
\midrule

\multicolumn{2}{@{}l}{\textit{Common optimization configuration}} \\[2pt]
\midrule
Optimizer                     & AdamW \\
Learning rate                 & 1e-6 \\
Global training batch size    & 16 \\
PPO mini-batch size           & 16 \\
Rollout group size            & 8 \\
Mixed-group composition       & 7 on-policy episodes + 1 joint/off-policy episode \\
Maximum prompt length         & 8192 \\
Maximum response length per turn   & 4096 \\
Maximum episode steps         & 20 \\
KL coefficient $\beta$        & 0 \\
Entropy coefficient           & 0 \\
PPO clipping $(\epsilon_{\mathrm{low}},\epsilon_{\mathrm{high}})$ & $(0.20,0.28)$ \\
Advantage normalization       & Disabled \\
Loss aggregation              & Seq-mean-token-mean \\
Sampling temperature          & 1 \\

\addlinespace[3pt]
\midrule
\multicolumn{2}{@{}l}{\textit{\Mymethod{} configuration}} \\[2pt]
\midrule
Joint-decoding policy         & Hard joint policy ($\lambda=0$) \\
Joint-decoding trigger        & All student episodes in the rollout group fail \\
Joint-decoding response length  & 2048 \\
Student cutoff $K_{\mathrm S}$  & 50 \\
Maximum teacher cutoff $K_{\mathrm T}^{\max}$  & 20 \\
Retained-mass threshold $Z_{\min}$   & 0.8 \\

\bottomrule
\end{tabularx}
\end{table*}

\begin{table}[t]
\centering
\caption{
Comparison of training efficiency for Qwen3-4B-Base on DAPO-Math-17K. All experiments use Qwen3-4B-Base and four NVIDIA H20 GPUs.
Time per step is averaged over all policy-update steps.
Rollout utilization is the percentage of sampled rollout groups with
non-zero group-relative advantages.
Utilization/time is calculated as rollout utilization divided by average
time per step and is reported in percentage points per minute
(higher is better).
}
\label{tab:math-training-efficiency}
\small
\setlength{\tabcolsep}{7pt}
\renewcommand{\arraystretch}{1.12}

\begin{tabular}{
    @{}
    l
    c
    c
    c
    @{}
}
\toprule
\textbf{Method}
& \textbf{Time/step (min)}
& \textbf{Rollout Utilization (\%)}
& \textbf{Utilization/time} \\
\midrule

GRPO& 28.3 &57.2 &2.01  \\
\rowcolor[HTML]{E8F3F1}
\textbf{\Mymethod{}}
& 43.6 & 91.9 & 2.11\\

\bottomrule
\end{tabular}
\end{table}

\subsection{Evaluation Details}
\label{evaluation-benchmarks}
\paragraph{Mathematical Reasoning Benchmarks.}
We evaluate mathematical reasoning on AIME 2025, AIME 2026, HMMT 2026, and MATH500. AIME 2025 and AIME 2026 each contain 30 problems drawn from the corresponding AIME I and AIME II competitions. These problems require exact numerical answers and cover algebra, geometry, combinatorics, and number theory. HMMT 2026 contains 33 problems from the February 2026 Harvard--MIT Mathematics Tournament and provides a complementary evaluation of advanced competition-level mathematical reasoning.  MATH-500 is a widely used subset of the MATH test set comprising 500 problems across seven mathematical subjects and a range of difficulty levels. Compared with the recent competition benchmarks, MATH-500 provides broader coverage of mathematical topics and evaluates performance on a larger and more diverse problem set. Since our reasoning models are trained on DAPO-Math-17K, we treat these three competition-mathematics benchmarks as in-domain evaluations.

\paragraph{Out-of-domain Reasoning Benchmarks.}
We evaluate cross-domain generalization on GPQA-Diamond and Humanity's Last Exam (HLE). GPQA-Diamond contains 198 expert-written, graduate-level multiple-choice questions, primarily covering physics, chemistry, and biology. The full HLE benchmark comprises 2,500 expert-level questions spanning more than 100 academic and professional subjects. We use its text-only subset, which contains 500 questions and excludes items that require multimodal inputs. These benchmarks differ substantially from the mathematics-focused training distribution and therefore assess whether the reasoning capabilities acquired during training transfer to previously unseen domains.

\paragraph{Agentic Benchmarks.}
We evaluate agentic reasoning on ALFWorld and ScienceWorld using Qwen3-4B-Instruct. ALFWorld is a text-based interactive benchmark in which an agent must complete household tasks by interpreting observations and executing sequences of admissible actions. Following the standard evaluation protocol, we evaluate on 140 task instances from the valid-seen split and 134 task instances from the valid-unseen split. ScienceWorld is an interactive text environment centered on elementary-science reasoning. It comprises 30 task types across 10 scientific topics, instantiated into 1,819 task variations. These tasks require the agent to explore the environment, manipulate objects, and complete multi-step scientific procedures. For both benchmarks, each rollout corresponds to a complete interaction episode, and we use the environment-provided task-completion signal to determine success.
\paragraph{Evaluation Hyperparameters.}
For all reasoning benchmarks, we independently sample 16 responses per problem and report both avg@16 and pass@16. We use a sampling temperature of $0.6$, top-$p$ of $1$, top-$k$ of $-1$, and a maximum response length of $20{,}480$ tokens. We apply the same prompt template, answer extractor, and verifier to all methods. For ALFWorld and ScienceWorld, we use a sampling temperature of 1 and a maximum of 20
interaction steps per episode. We evaluate each task instance over 8 episode(s) and report the mean episode success rate. During all evaluations, the student operates independently from the original task input without expert solutions, teacher predictions, or joint decoding.

\section{Prompts}
\label{appendix-prompts}
we present the prompts used for mathematical reasoning and agentic tasks. The student prompt contains only the original task information and the interaction context available under the target policy. The teacher receives the same task context together with privileged guidance derived from off-policy data. For mathematical reasoning, this guidance is an expert solution associated with the training problem. For ALFWorld, it is a known-correct expert trajectory, which the teacher aligns with the current observation and action history to identify a suitable next action. The off-policy data is never included in the student's prompt or generated context. Instead, the teacher uses it only to produce token-level predictions during joint decoding, while the student distribution determines the relative sampling probabilities over the retained candidate tokens. This design incorporates information from successful off-policy solutions without training the student directly on their full trajectories or exposing privileged information at inference time.

\subsection{Mathematical Reasoning Prompt}

\begin{promptbox}{Mathematical Reasoning Prompt}
You are a helpful assistant. Please reason step by step, and put your final answer within \boxed{}.

{problem}
\end{promptbox}

\begin{promptbox}{Teacher Prompt}
You are an expert mathematician. Use the mathematical reasoning provided in <private_guidance> to help solve the problem. Write one complete, self-contained solution in a natural step-by-step problem-solving style. Include the necessary equations, intermediate conclusions, and justifications, fill in any missing steps, and ensure that the final result is mathematically correct. Present only the direct solution and put the final answer within \boxed{}.

<problem>
{question}
</problem>

<private_guidance>
{prefix}
</private_guidance>

## Solution:
\end{promptbox}

\newpage
\subsection{ALFWorld Prompts}
\label{appendix-alfworld-prompts}

\begin{agentpromptbox}{ALFWorld Reasoning Prompt}
<System Prompt>

You are acting in ALFWorld TextWorld. Choose exactly one command from the provided admissible commands on each turn. You may reason briefly inside <think>...</think> before acting. Call the env_step tool with that exact command. Do not output a final natural-language answer.

<User Prompt>

### Task
{task_text}

### Current Observation
{observation}

### History Actions
{history_actions}

### Admissible Commands
{admissible_commands}

### Instructions
- Think briefly about the current state inside <think>...</think> tags.
- Select exactly one command and execute it through the env_step tool.
- The command must exactly match one item from Admissible Commands.
- Follow the ALFWorld TextWorld command style, such as `go to dresser 1`, `take mug 1 from cabinet 3`, or `use desklamp 1`.
- Treat the official observation and admissible-command list as the source of truth.
- Do not output explanations or a final natural-language answer.

### Output Format
<think>
[Brief reasoning about the current state and next command.]
</think>
<tool_call>
{{"name": "env_step", "arguments": {{"command": "[one admissible command]"}}}}
</tool_call>
\end{agentpromptbox}

\newpage
\begin{agentpromptbox}{ALFWorld Teacher Prompt}
<System Prompt>

You are acting in ALFWorld TextWorld. Choose exactly one command from the provided admissible commands on each turn. You may reason briefly inside <think>...</think> before acting. Call the env_step tool with that exact command. Do not output a final natural-language answer.

<User Prompt>

### Task
{task_text}

### Current Observation
{observation}

### History Actions
{history_actions}

### Expert Guidance
The following expert trajectory is a known-correct solution for the task and is provided in full from the initial state:

{expert_trajectory}

Use the expert trajectory as private guidance. Locate the current progress by aligning History Actions and Current Observation with the trajectory, and identify the next action along the corresponding solution path. Object names may differ from those in the expert trajectory, so match entities and actions by their semantic roles. Do not copy an expert action when it conflicts with the current observation or is absent from Admissible Commands.

### Admissible Commands
{admissible_commands}

### Instructions
- Think briefly about the current state inside <think>...</think> tags.
- Select exactly one command and execute it through the env_step tool.
- The command must exactly match one item from Admissible Commands.
- Follow the ALFWorld TextWorld command style, such as `go to dresser 1`, `take mug 1 from cabinet 3`, or `use desklamp 1`.
- Treat the official observation and admissible-command list as the source of truth.
- Do not output explanations or a final natural-language answer.

### Output Format
<think>
[Brief reasoning about the current state and next command.]
</think>
<tool_call>
{{"name": "env_step", "arguments": {{"command": "[one admissible command]"}}}}
</tool_call>
\end{agentpromptbox}

\end{document}